\documentclass[11pt,a4paper]{article}
\usepackage[utf8]{inputenc} 
\usepackage[T1]{fontenc}

\usepackage{tabularx}
\usepackage[numbers]{natbib}
\usepackage[pdftex]{graphicx} 
\usepackage[pdftex,linkcolor=black,pdfborder={0 0 0}]{hyperref} 
\usepackage{calc} 
\usepackage{enumitem} 
\usepackage{amsmath}
\usepackage{authblk}
\usepackage{amsthm}
\usepackage{amssymb}
\usepackage{subcaption}
\usepackage{algpseudocode}
\usepackage[linesnumbered,ruled,vlined]{algorithm2e}
\usepackage{float}
\usepackage{booktabs}
\usepackage{adjustbox}
\usepackage{authblk}
\usepackage{multirow}
\usepackage{booktabs}  
\usepackage{makecell}  
\usepackage[english]{babel}

\newtheorem{theorem}{Theorem}[section]

\newtheorem{proposition}[theorem]{Proposition}
\newtheorem{corollary}[theorem]{Corollary}

\newtheorem{example}[theorem]{Example}

\usepackage[a4paper, lmargin=0.1666\paperwidth, rmargin=0.1666\paperwidth, tmargin=0.1111\paperheight, bmargin=0.1111\paperheight]{geometry} 

\usepackage[all]{nowidow} 
\usepackage[protrusion=true,expansion=true]{microtype} 

\usepackage{lipsum} 

\title{Neural Certificate Pricing for Combinatorial Optimization Problems}

\author[1,$\dagger$]{Jingyi Chen}
\author[1,$\dagger$]{Xinyuan Zhang}
\author[1,*]{Xinwu Qian}

\affil[1]{Department of Civil and Environmental Engineering, Rice University, Houston, TX}
\affil[$\dagger$]{These authors contributed equally to this work}
\affil[*]{Corresponding author}

\date{\today}
\begin{document}
\maketitle

\begin{abstract}
Combinatorial optimization (CO) problems are difficult because certifiable discrete structure induces exponential search. One needs to search over the set exponentially many candidates to certify optimality, however, the structural feasibility of a path, packing, or cover can be verified in polynomial time once supplied. In this study, we introduce Neural Certificate Pricing (NCP) that exploits this asymmetry under an unsupervised learning framework. A neural network is trained to predict certificate-level dual prices, while a structured recovery layer constructs the induced primal marginal. NCP can be viewed as amortized separation: instead of enumerating violated inequalities, it learns the residual prices through which their aggregate effect enters recovery. When the certificate-consistency condition holds, the recovered marginal is globally feasible, and a local theory shows that first-order errors in the predicted price induce only second-order loss in objective value. Across three classes of CO problems, NCP either outperforms state-of-the-art neural baselines by large margins or matches them at a fraction of the computation time, and shows stronger out-of-distribution generalization.
\end{abstract}

\section{Introduction}



Combinatorial optimization (CO) on graphs underlies decision-making across a wide range of networked systems, where local choices on nodes and edges are assembled into global structures whose quality determines system performance. The difficulty has two coupled sources: the combinatorial structure produces exponentially many candidate objects from local choices, and coupling constraints render these local choices non-separable within any valid global object. Classical CO confronts this difficulty through structured search over the exponential family, examples include bounding via branch-and-bound~\cite{lawler1966branch}, pricing via column generation~\cite{desaulniers2006column}, and separation via cutting planes~\cite{kelley1960cutting}. These CO problems, however, exhibit a less-exploited asymmetry between search and certification (the one that also defines NP): a candidate solution is a discrete object whose feasibility and objective value are verifiable in polynomial time, while finding the optimizing one is not. In this study, we ask:

\vspace{-0.25em}
\begin{center}
\emph{Can we solve CO by turning certification from an after-the-fact check into a mechanism for recovering optimal solutions?}
\end{center}
\vspace{-0.25em}

Recent learning-based approaches to CO have explored many ways to inject learning into the solution pipeline, including supervised imitation of optimal solutions or solver decisions~\citep{gasse2019exact,li2018combinatorial}, reinforcement-learning construction policies~\citep{bello2016neural, khalil2017learning,kool2018attention}, and unsupervised losses over relaxed or decoded combinatorial objects~\citep{karalias2020erdos,wang2019satnet}. Despite their differences, these methods are largely object-facing or search-facing: learning predicts scores, heatmaps, policies, relaxed marginals, or solver guidance for navigating the combinatorial structure. The certificate that makes a recovered object globally valid is still supplied downstream by decoding, repair and local search. 

\begin{figure}[t]
    \centering
    \resizebox{\linewidth}{!}{%
        \includegraphics[height=3.0cm]{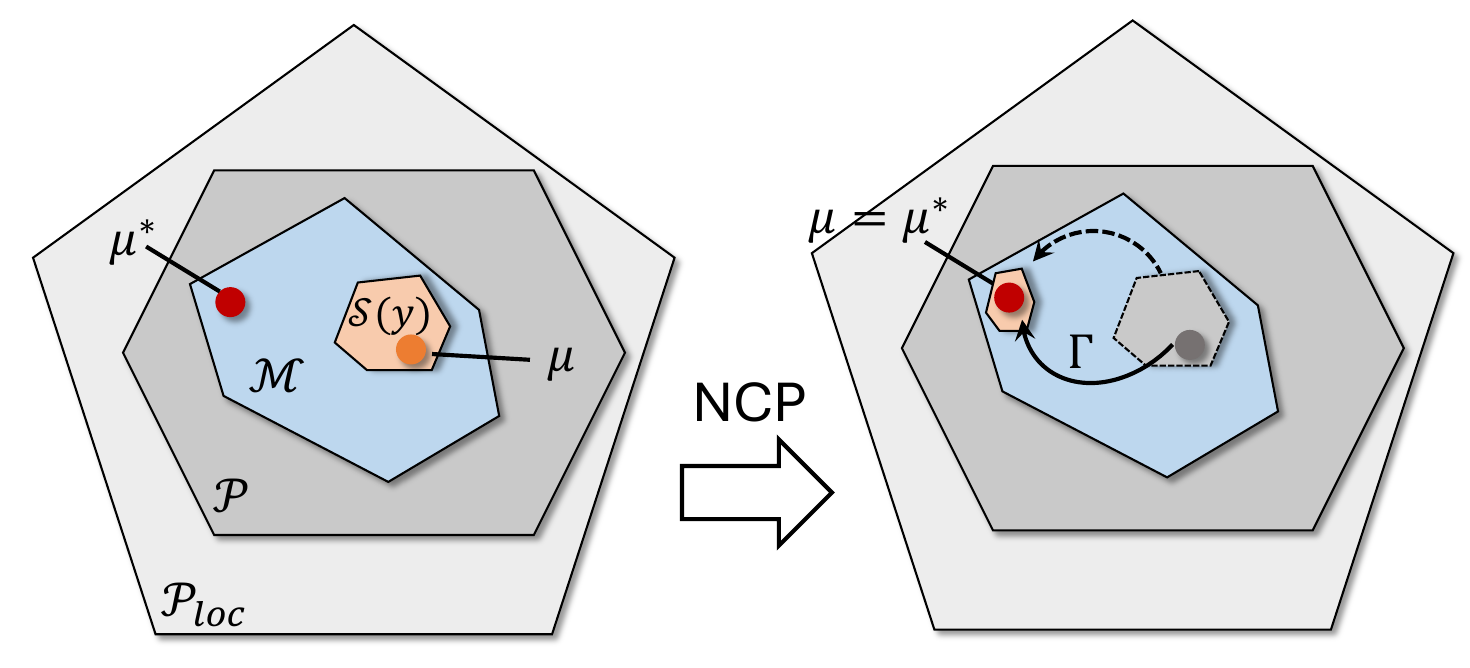}\hspace{0.5cm}
        \includegraphics[height=2.5cm]{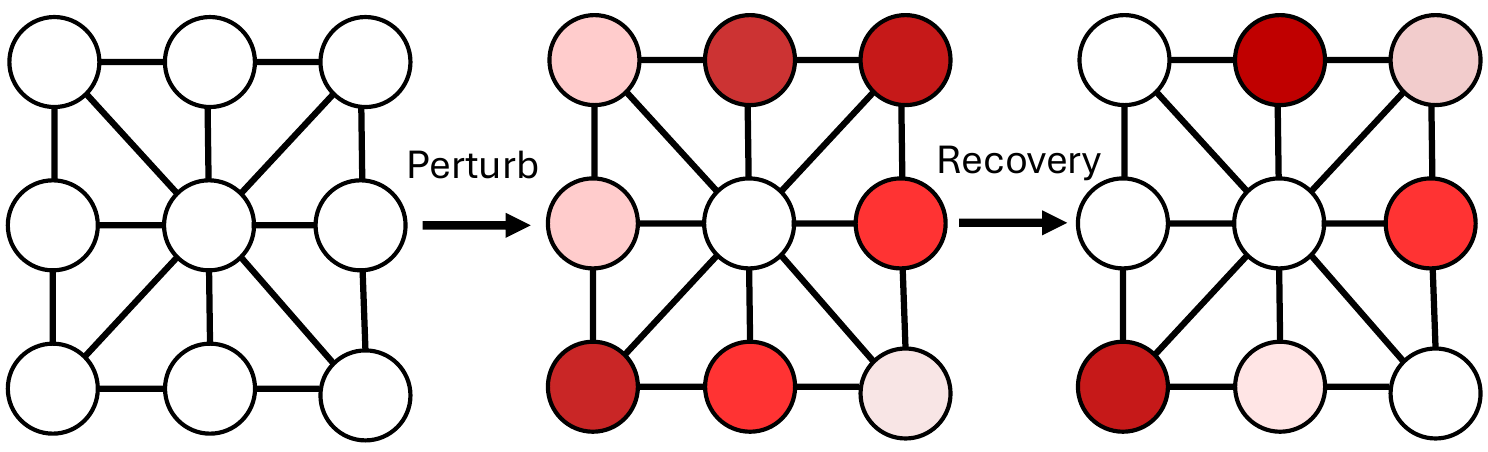}
    }
    \caption{\textbf{Left}: Schematic illustration of NCP. NCP optimizes the problem inside $\mathcal M_G$ and uses perturbation $\Gamma$ (price signal) to move the marginal $\mu$ to $\mu^*$ while being consistent to the structure certificate $S(y)$.
\textbf{Right:} Maximum-independent set instance. NCP perturbs local node values, then recovers a certificate-consistent marginal from which an independent set can be recovered.}
    \label{fig:NCP_overview}
\end{figure}

In this paper, we introduce neural certificate pricing (NCP), an unsupervised framework that turns certificate-guided recovery into a learnable optimization pipeline. We represent feasible discrete objects in marginal space~\citep{wainwright2008graphical}, where local consistency, certificate-induced structure, and retained coupling constraints can be separated. NCP assigns learning to the missing price signal. A neural network predicts prices, and a structured recovery layer uses these prices to recover a certificate-consistent primal marginal. The network therefore does not directly predict or approximate the solution, but instead supplies the price signal under which certificate-conditioned recovery exposes a valid structure. Training then evaluates the original primal objective directly on the recovered marginal, yielding an unsupervised learning-optimization framework without optimal labels. Figure~\ref{fig:NCP_overview} illustrates the NCP pipeline with one representative instantiation.

Our contributions are fourfold. First, we introduce NCP as a novel certificate-based framework that learns pricing signals for structured recovery to solve challenging CO problems. Second, we develop an oracle certificate-pricing formulation that separates the learned residual price signal, the structural certificate that induces tractable recovery, and the dual certificate for retained coupling constraints. Third, we prove a local stability result showing that certificate-consistent recovery behaves like a smooth price-perturbation problem, and that near a primal-realizing point, small price errors cause only second-order degradation in objective value. Fourth, we validate NCP through instantiations on three CO problem classes and show substantial and consistent improvements over neural and non-learning heuristic baselines.

\section{Related Works}
\paragraph{Marginal-polytope inner approximations.}
The marginal polytope provides a unified framework to inference and optimization over discrete combinatorial families. Its facet description, however, is exponential in problem size and intractable in the regimes of interest~\citep{wainwright2008graphical}. The classical response is to replace the marginal polytope by an analytically fixed approximation. Tree-reweighted bounds combine tree-structured distributions on a modeler-chosen tree set~\citep{wainwright2005new}. Bethe and Kikuchi methods relax to the local polytope with a closed-form entropy surrogate~\citep{yedidia2005constructing}. Cutting-plane tightening adds valid inequalities from a separation oracle to close the gap between local and the marginal polytopes~\citep{sontag2012tightening}. Differentiable structured-prediction layers expose a fixed structured polytope as an end-to-end trainable component~\citep{niculae2018sparsemap}. In each case the inner approximation is fixed before training, independent of the cost vector. Our framework departs from this pattern by treating the inner approximation itself as the learned object.

\paragraph{Learning for solving combinatorial optimization.}
Machine learning for CO has developed along supervised, reinforcement-learning, and unsupervised directions; see \citep{bengio2021machine,smit2025graph,da2025large,schiffer2026combinatorial,
chung2025neural,zhou2025learning} for surveys. Majority of the ML for CO study targets approximation the final optimal solutions. Supervised methods imitate optimal or high-quality solutions, as in Pointer Networks for routing and GNN-based search policies for graph problems~\citep{vinyals2015pointer,li2018combinatorial,joshi2019efficient}, but require expensive solved instances. Reinforcement-learning methods remove the need for labels by optimizing rewards over sampled solutions \citep{bello2016neural,khalil2017learning,kool2018attention,kwon2020pomo}. Unsupervised methods train from structural losses or probabilistic objectives, including neural probabilistic-method objectives, recurrent GNNs for CSPs, and diffusion-based samplers over discrete solution spaces \citep{karalias2020erdos,toenshoff2019run,sanokowski2024diffusion}. NCP also uses unsupervised training, but its target is different: the network learns a perturbation as the pricing signal, and the recovery layer then converts into a certificate-consistent marginal evaluated under the original primal objective.

\paragraph{Learning to shape relaxations and search.}
A related line does not predict a final solution directly, but learns signals that guide a downstream optimization procedure. Branching and cutting policies use graph neural networks to guide mixed-integer solvers \citep{gasse2019exact}. Other methods learn graph modifications, search heuristics, or iterative refinement rules whose outputs are subsequently decoded or repaired by a combinatorial procedure \citep{wang2021bi,acikalin2025learning}. Primal-dual message-passing methods also use optimization structure to shape neural updates \citep{he2025primal}. These methods show the value of learning algorithmic guidance rather than complete solutions. NCP follows this broader direction, but with a different downstream object. Prior methods supply guidance to a search or repair procedure that explores discrete candidates, while NCP supplies a perturbation to a certificate-conditioned recovery layer that produces a feasible marginal in one shot.

\section{Neural Certificate Pricing}
\label{sec:framework}

We study a broad class of graph-based CO problems. For convenience, we use edge-indexed configurations for notation, and the framework applies directly to other decision-indexed configurations such as nodes and assignment pairs. Let $G=(V,E,c)$ denote an instance with edge costs $c\in\mathbb R^{|E|}$, and let $\mathcal X_G\subseteq 2^E$ be the set of feasible edge configurations. Each configuration $x\in\mathcal X_G$ is represented by its incidence vector $\phi(x)=(\mathbf 1\{e\in x\})_{e\in E}$. Following~\cite{wainwright2008graphical}, the exact marginal polytope is
\[
    \mathcal M_G=\operatorname{conv}\{\phi(x):x\in\mathcal X_G\}.
\]
A point $\mu\in\mathcal M_G$ is a feasible mean parameter, with $\mu_e$ equal to the marginal probability that edge $e$ appears under some distribution over $\mathcal X_G$. Using the marginal polytope view, the original problem can be written as a free-energy problem:
\begin{equation}
    p^\star(G)
    =
    \min_{\mu\in\mathcal M_G}
    F_G(\mu),
    \qquad
    F_G(\mu)
    =
    \langle c,\mu\rangle-\tau H_{\mathcal M}(\mu),
    \tag{P}
    \label{eq:primal}
\end{equation}
where $\tau\ge 0$ and $H_{\mathcal M}(\mu)=\sup_{q:\mathbb E_q[\phi]=\mu}H(q)$ is the maximum entropy among distributions over $\mathcal X_G$ with mean $\mu$. When $\tau=0$, \eqref{eq:primal} recovers the original combinatorial objective. When $\tau>0$, it gives a smooth free-energy problem on the interior of $\mathcal M_G$. The full derivation is given in Appendix~\ref{app:kl_derivation}. 

In CO, optimization over $\mathcal M_G$ is intractable and one usually begins with a relaxed polytope 
\begin{equation}
    \mathcal P_G
    :=
    \{\mu\in\mathcal P_{\mathrm{loc}}(G): A_G\mu\le b_G\},
    \qquad
    \mathcal M_G\subseteq \mathcal P_G,
    \label{eq:local_global}
\end{equation}
where $\mathcal P_{\mathrm{loc}}(G)$ is the polytope of local constraints (i.e., flow conservation at local nodes), and $A_G\mu\le b_G$ are the explicit global coupling constraints (i.e., capacity constraints shared by all variables). In general, this relaxation is loose and $\mathcal P_G$ contains relaxation-feasible mean parameters outside $\mathcal M_G$. Classical cutting-plane methods~\cite{kelley1960cutting} can close this gap but requires separation over potentially exponentially many valid inequalities.

\subsection{Certificate-conditioned recovery under perturbation}
\label{subsec:local_recovery}
The core idea of neural certificate pricing (NCP) is to replace direct optimization over $\mathcal M_G$ by tractable recovery over certificate-induced structured sets. The key asymmetry is that characterizing the full $\mathcal M_G$ can be intractable, while certifying a $\mu \in \mathcal M_G$ can be much easier. Specifically, marginal $\mu \in \mathcal M_G$ once it is realized as a convex combination of feasible discrete configurations. NCP then exploits this fact by learning certificates that induce such realizable recovery sets.

\paragraph{Certificate-conditioned structure recovery map.}
We assume the existence of a certificate $y\in\mathcal Y_G$ that indexes a tractable recovery set $\mathcal S_G(y)\subseteq\mathcal P_{\mathrm{loc}}(G)$. Equivalently, one can view $y$ as the pricing signal that can be turned into the set of (convex combinations of) structured patterns $\mathcal S_G(y)$, such as an ordering, assignment, or conflict patterns. For feasible discrete objects and their convex combinations, $\mathcal P_{loc}$ automatically holds, and the recovery set maps to:
\begin{equation}
    \mathcal S_G(y)\cap \mathcal P_G \subseteq \mathcal M_G,
    \qquad
    \forall y\in\mathcal Y_G.
    \label{eq:structure_preserving}
\end{equation}
When there are no explicit coupling constraints $A_Gu\leq b_G$, Eq.~\ref{eq:structure_preserving} reduces to $\mathcal S_G(y)\subseteq\mathcal M_G$, and any $y$ can induce a feasible recovery set. When explicit coupling constraints are present, we can use their dual price to induce a valid certificate. 

\begin{example}
The above tractable recovery structure is widely available in many CO problems. 
In routing, node potentials ($y$) induce an acyclic order of nodes ($S(y)$), so path recovery can be carried out by a Bellman-style dynamic program (polynomial). In assignment, reduced-cost structure restricts each task to a small set of candidate machines, so recovery operates on a tractable local assignment subproblem. For packing problems, conflict or clique covers
replace pairwise edge constraints by stronger cover constraints while keeping the
description tractable. 
\end{example}

With the above preliminaries, we next introduce the general recovery map function with the presence of global coupling constraints. We denote $\lambda\in\mathbb R_+^{m_G}$ the dual variable  prices the $A_G\mu\le b_G$, and we introduce the perturbation $\Gamma\in\mathbb R^{E}$ to augment the dual prices. Given fixed $(\Gamma,y,\lambda)$, the primal recovery layer can be written as
\begin{equation}
    Q_G(\Gamma,y,\lambda)
    =
    \arg\min_{\mu\in\mathcal S_G(y)}
    \left\{
        F_G(\mu)
        +
        \langle \Gamma + A_G^\top\lambda,\mu\rangle
    \right\}. 
    \label{eq:local_recovery}
    \tag{PR}
\end{equation}
If a stronger condition $S_G(y)\subseteq \mathcal M_G$ holds, which we formally establish in the next section, then the recovery map is feasibility-preserving for every $(\Gamma,y,\lambda)$, and each $\hat{\mu}\in Q$ induces a primal bound $F_G(\hat\mu)$ for~\eqref{eq:primal}.Central to this recovery map is the role of perturbation $\Gamma$. When $\Gamma=0$, \ref{eq:local_recovery} reduces to the canonical Lagrange relaxation~\cite{geoffrion1974}, but is subject to a weak primal bound depending on the tightness of its Lagrange relaxation. $\Gamma$ hence supplies the missing price action that allows one to optimize beyond the original Lagrange space, and is therefore the learning target in NCP. We formalize the sufficiency condition of a linear perturbed $\Gamma$ for recovering optimality in proposition~\ref{prop:linear_sufficiency}. 
\begin{proposition}[Sufficiency of linear cost perturbation]
\label{prop:linear_sufficiency}
Fix a certificate $y$ such that an optimizer
$\mu^*$ belongs to $S_G(y)$.
Suppose the valid inequalities omitted from the recovery layer in certifying
$\mu^*$ can be written as $B_G\mu\le d_G$, and let $(\lambda^*,z^*)$ be KKT
multipliers at $\mu^*$ for minimizing $F_G$ over $S_G(y)$ subject to
$A_G\mu\le b_G$ and $B_G\mu\le d_G$. Then,
\[
\Gamma^*:=B_G^\top z^*
\quad\Longrightarrow\quad
\mu^*\in Q_G(\Gamma^*,y,\lambda^*).
\]
\end{proposition}

\subsection{Certificate consistency and primal realization}
\label{subsec:certificate_consistency}
The structure recovery map in~\eqref{eq:local_recovery} gives a tractable response once $(\Gamma,y,\lambda)$ is fixed. But this alone does not certify that the response is a valid marginal satisfying (Eq~(\ref{eq:structure_preserving}). Here we establish the remaining consistency condition that ties the recovered marginal back to $\mathcal M_G$. 

Let $C_G$ be the certificate readout map and let $\hat\mu = Q_G(\Gamma,y,\lambda)$. We call $(y,\lambda)$ a \emph{consistent certificate at perturbation $\Gamma$}
if
\begin{equation}
    y=C_G(\hat\mu),
    \qquad
    A_G\hat\mu\le b_G,
    \qquad
    \lambda\ge0,
    \qquad
    \lambda^\top(A_G\hat\mu-b_G)=0 .
    \label{eq:certificate_consistency_kkt}
\end{equation}
The first condition enforces structural self-consistency at $\hat{\mu}$, while the remaining conditions are the KKT complementarity conditions for the explicit global coupling constraints $A_G\mu\le b_G$. The following proposition then formalizes the consistency of the recovery map under perturbation. 

\begin{proposition}[Feasibility and shifted optimality of consistent recovery]
\label{prop:consistent_recovery}
Let $\hat\mu=Q_G(\Gamma,y,\lambda)$ and suppose $(y,\lambda)$ is a consistent
certificate at perturbation $\Gamma$. Then $\hat\mu\in\mathcal M_G$.
Moreover, for every $\nu\in\mathcal S_G(y)\cap\mathcal P_G$,
\begin{equation}
    F_G(\hat\mu)-F_G(\nu)
    \le
    \langle \Gamma,\nu-\hat\mu\rangle .
    \label{eq:shifted_recovery_gap}
\end{equation}
\end{proposition}

The proof is given in Appendix~\ref{app:proof_consistent_recovery}.
Proposition~\ref{prop:consistent_recovery} turns consistency into a certificate: once $(y,\lambda)$ is consistent at $\Gamma$, the recovered point is a valid marginal and satisfies the optimality condition of the shifted structured problem. 


\paragraph{Oracle equilibrium-constrained characterization.}

Proposition~\ref{prop:consistent_recovery} gives an oracle characterization of NCP:
\begin{equation}
\begin{aligned}
    \min_{\Gamma,\,y,\,\lambda,\,\mu}\quad
    & F_G(\mu) \\
    \mathrm{s.t.}\quad
    & y=C_G(\mu),\qquad \mu \in Q_G(\Gamma,y,\lambda), \qquad
      0\le \lambda \perp b_G-A_G\mu \ge 0
\end{aligned}
\tag{OC}
\label{eq:oracle_certificate_mpec}
\end{equation}
Here $0\le \lambda \perp b_G-A_G\mu\ge0$ denotes primal feasibility, dual feasibility, and complementary slackness. Problem~\eqref{eq:oracle_certificate_mpec} is a bi-level problem in which the outer level chooses $\Gamma$ to minimize the recovered value $F_G(\mu)$, and the inner level recovers a feasible marginal through the certificate-consistency equilibrium, consisting of structured recovery and complementarity for the explicit coupling constraints. Equivalently, after replacing the inner recovery problem by its optimality system, the formulation is a mathematical program with equilibrium constraints. It therefore provides an oracle characterization for designing $\Gamma$. Since any solution yields a valid marginal certificate, for the original linear CO objective, the associated cost $\langle c,\mu\rangle$ gives a valid primal bound because $\mu\in\mathcal M_G$. Finally, when $\mathcal S_G(y)=\mathcal P_{\mathrm{loc}}(G)$ and the certificate has no structural role, \eqref{eq:oracle_certificate_mpec} reduces to ordinary relaxed primal-dual recovery as a special case.

\subsection{Unsupervised learning for NCP}
\label{subsec:unsupervised_learning}

The oracle characterization in~\eqref{eq:oracle_certificate_mpec} identifies $\Gamma$ as the residual perturbation to be optimized. NCP amortizes this outer variable by predicting $\Gamma_\theta(G)\in\mathbb R^{|E|}$ from the instance $G$, followed by a fixed-point process to solve for the certificate consistency condition. Specifically, for a given $(G,\theta)$, write $z=(y,\lambda)$ and define the projected dual update for the fixed-point process as:
\[
    D_\rho(\lambda,\mu)
    =
    \Pi_{\mathbb R_+^{m_G}}
    [\lambda+\rho(A_G\mu-b_G)] .
\]
Starting from $z^0=(y^0,\lambda^0)$, the forward pass applies the damped projected certificate iteration
\begin{equation}
    \mu^{t+1}=Q_G(\Gamma_\theta(G),y^t,\lambda^t),\quad
    \bar z^{t+1}=\bigl(C_G(\mu^{t+1}),D_\rho(\lambda^t,\mu^{t+1})\bigr),\quad
    z^{t+1}=(1-\alpha)z^t+\alpha\bar z^{t+1},
    \label{eq:NCP_forward_iteration}
\end{equation}
where $\rho>0$ and $\alpha\in(0,1]$. The two components of $\bar z^{t+1}$ update the structure certificate and the dual certificate, respectively. Damping stabilizes the forward solver without changing its fixed points. At convergence, the selected state $z^\star=(y^\star,\lambda^\star)$ satisfies the certificate consistency conditions, and the recovered marginal is
\begin{equation}
    \mu_\theta(G)=Q_G(\Gamma_\theta(G),y^\star,\lambda^\star).
    \label{eq:NCP_recovered_marginal}
\end{equation}
The unsupervised objective then evaluates the unperturbed value of this recovered marginal:
\begin{equation}
    \min_\theta \mathcal L_{\mathrm{NCP}}(\theta)=\mathbb E_{G\sim\mathcal D}\!\left[F_G(\mu_\theta(G))\right].
    \label{eq:NCP_unsupervised_objective}
\end{equation}
This objective trains the network from the value of the certificate-consistent marginal induced by its predicted perturbation. The perturbation is used only inside recovery and certificate consistency, while the loss is measured under the original primal objective. We compute gradients by implicit differentiation~\cite{bai2019deep} through the certificate-consistency fixed point, which avoids backpropagating through long recovery iterations. The implicit-gradient formula and its regularity conditions are given in Appendix~\ref{app:implicit_gradient}.

\section{Theoretical Analysis}
\label{sec:theory}

\eqref{eq:oracle_certificate_mpec} casts NCP as an equilibrium-constrained problem. This may raise the concern that we have relocated combinatorial difficulty into a generic MPEC, a class notoriously hard to optimize~\cite{luo1996mathematical}. In this section, we address this concern by showing favorable local structure of \eqref{eq:oracle_certificate_mpec}, where it can be reduced locally to an unconstrained smooth problem in the perturbation object \(\Gamma\). 

Write \(z=(y,\lambda)\) and \(\hat\mu_G(z,\Gamma):=Q_G(\Gamma,y,\lambda)\), and fix a reference pair \((\bar z,\bar\Gamma)\) at which OC is solved. At
this point, partition the explicit constraints \(A_G\hat\mu_G\le b_G\) into active and inactive blocks, labeled \(\mathcal A\) and \(\mathcal I\):
\[
    A_G=\begin{bmatrix}A_{\mathcal A}\\ A_{\mathcal I}\end{bmatrix},
    \qquad
    b_G=\begin{bmatrix}b_{\mathcal A}\\ b_{\mathcal I}\end{bmatrix},
\]
with \(A_{\mathcal A}\hat\mu_G-b_{\mathcal A}=0\) and \(A_{\mathcal I}\hat\mu_G-b_{\mathcal I}<0\), and their dual variables \(\lambda_{\mathcal A}\) and \(\lambda_{\mathcal I}\), respectively. At \((\bar z,\bar\Gamma)\), certificate consistency and complementarity can be compactly represented as the equality system
\begin{equation}
    \mathcal K_G(z,\Gamma)
    =
    \begin{bmatrix}
        y-C_G(\hat\mu_G(z,\Gamma))\\[2pt]
        A_{\mathcal A}\hat\mu_G(z,\Gamma)-b_{\mathcal A}\\[2pt]
        \lambda_{\mathcal I}
    \end{bmatrix}
    =0.
    \label{eq:certificate_kkt_residual}
\end{equation}

We want to extend OC's solution from a single point at \(\bar\Gamma\) to a smooth map on a neighborhood. This requires (a) the active set on \(A_G\mu\le b_G\) to be stable under perturbation, and (b) the equality system \(\mathcal K_G(z,\Gamma)=0\) to be locally invertible in \(z\). For the NCP construction, both hold mildly. (a) requires strict complementarity~\cite{robinson1980strongly} to hold only for explicit global coupling constraints $A_G$, which is the typical non-degenerate case with strictly positive shadow prices for binding constraints, and slack capacities carry zero multipliers.  (b) follows from positive-definiteness of the entropic Hessian at \(\tau>0\), linear independence of the rows of \(A_{\mathcal A}\), and contractivity of the damped certificate iteration. We refer to a point \((\bar z,\bar\Gamma)\) at which (a) and (b) hold as a
\emph{regular OC solution}. At any such point, the implicit function
theorem~\cite{krantz2002implicit} yields a locally unique \(C^2\)
solution map \(\Gamma\mapsto z_\Gamma\) on a neighborhood \(U\) of
\(\bar\Gamma\), with \(z_{\bar\Gamma}=\bar z\), and OC reduces on \(U\)
to the unconstrained smooth problem
\begin{equation}
    \min_{\Gamma\in U}\Phi_G(\Gamma),
    \qquad
    \Phi_G(\Gamma)
    :=
    F_G\bigl(\hat\mu_G(z_\Gamma,\Gamma)\bigr),
    \label{eq:reduced_oc_objective}
\end{equation}
where \(F_G\) is the unperturbed objective and \(\Gamma\) enters only through the recovered marginal \(\hat\mu_G(z_\Gamma,\Gamma)\).

With the smooth reduction in hand, we can now characterize OC's stationary points directly through the reduced objective. Theorem~\ref{thm:strong-stationarity} establishes the key equivalence on stationary points of \(\Phi_G\) are exactly KKT stationary points of OC in its local equality-constrained form. The reduction therefore preserves OC's first-order structure through the smooth reduced problem.

\begin{theorem}[Reduced stationarity of OC]
\label{thm:strong-stationarity}
At a regular OC solution, the implicit reduction defines a \(C^2\)
objective \(\Phi_G\) on \(U\). Moreover,
\[
    \nabla_\Gamma\Phi_G(\bar\Gamma)=0
\]
if and only if \((\bar z,\bar\Gamma)\) satisfies the KKT stationarity
conditions of the local equality-constrained OC problem
\[
    \min_{z,\Gamma}
    F_G(\hat\mu_G(z,\Gamma))
    \quad
    \mathrm{s.t.}
    \quad
    \mathcal K_G(z,\Gamma)=0 .
\]
\end{theorem}
The proof can be found in Appendix~\ref{app:optimization-proofs}. Note that \(\Phi_G\) is precisely the per-instance object that the unsupervised loss \(\mathcal L_{\mathrm{NCP}}\) minimizes, with the amortization \(\Gamma=\Gamma_\theta(G)\) supplying the outer parameterization across instances. The theorem therefore characterizes the optimization landscape that the network sees through training. More importantly, the smoothness of \(\Phi_G\) gives a desirable robustness property~\cite{bonnans2013perturbation}: even when the certificate state drifts under perturbation, the outer objective value remains close to the stationary value. We characterize this results in Theorem~\ref{thm:second-order-stability} below. 

\begin{theorem}[Second-order value stability]
\label{thm:second-order-stability}
Suppose the point \((\bar z,\bar\Gamma)\) is a regular OC solution and \(\nabla_\Gamma\Phi_G(\bar\Gamma)=0\). For any
\(\Gamma_\varepsilon\in U\) with \(\Gamma_\varepsilon-\bar\Gamma=O(\varepsilon)\),
\[
    z_{\Gamma_\varepsilon}-\bar z=O(\varepsilon),
    \qquad
    \hat\mu_G(z_{\Gamma_\varepsilon},\Gamma_\varepsilon)
    -
    \hat\mu_G(\bar z,\bar\Gamma)
    =
    O(\varepsilon),
\]
and
\[
    \Phi_G(\Gamma_\varepsilon)-\Phi_G(\bar\Gamma)
    =
    O(\varepsilon^2).
\]
\end{theorem}

Theorems~\ref{thm:strong-stationarity} and~\ref{thm:second-order-stability} together characterize the favorable local structure of \(\Phi_G\) on \(U\), and the remaining gap is to bridge the understanding of \(\Phi_G\) to the
true optimum \(p^\star(G)\) of the underlying marginal problem. The corollary supplies this relation.

\begin{corollary}[Suboptimality decomposition]
\label{cor:objective-exact-second-order}
Let $(\bar z,\bar\Gamma)$ be a regular OC solution such that
$\bar\Gamma$ is a stationary point of $\Phi_G$ and $\bar z=z_{\bar\Gamma}$,
and suppose the structure-preservation condition~\eqref{eq:structure_preserving} holds on $U$. Define
\[
\Delta_{\mathrm{cert}}(G,\bar\Gamma)
:=
\Phi_G(\bar\Gamma)-p^*(G),
\]
where $p^*(G)$ is the optimal value of the original unperturbed marginal
problem. Then, for any $\Gamma_\epsilon\in U$ with
$\Gamma_\epsilon-\bar\Gamma=O(\epsilon)$,
\[
\Phi_G(\Gamma_\epsilon)-p^*(G)
=
\Delta_{\mathrm{cert}}(G,\bar\Gamma)+O(\epsilon^2).
\]
If, in addition, $\bar\Gamma$ satisfies the sufficient condition of
Proposition~\ref{prop:linear_sufficiency}, then $\Delta_{\mathrm{cert}}(G,\bar\Gamma)=0$, and
\[
0\le \Phi_G(\Gamma_\epsilon)-p^*(G)=O(\epsilon^2).
\]
\end{corollary}

The decomposition identifies two qualitatively different sources of suboptimality. The certificate gap \(\Delta_{\mathrm{cert}}\) is structural, depending on whether the certificate family \((\mathcal S_G(y), C_G)\) and the admissible perturbations can represent the optimum. The certificate must be compatible with the discrete object, since a clique certificate cannot encode a path object. When this compatibility holds and Proposition~\ref{prop:linear_sufficiency}'s condition is satisfied, \(\Delta_{\mathrm{cert}}\) vanishes. The \(O(\varepsilon^2)\) term is the local optimization residual around a stationary perturbation, which training is positioned to shrink. Proposition~\ref{prop:linear_sufficiency} identifies an ideal perturbation \(\Gamma^\star = B_G^\top z^\star\), but predicting \(z^\star\) accurately is essentially as hard as characterizing the exponentially many omitted valid inequalities. The corollary shows that this is not restrictive: model training on \(\mathcal L_{\mathrm{NCP}}\) only needs to reach a neighborhood of a stationary point of \(\Phi_G\), where second-order stability ensures an \(O(\varepsilon^2)\) value residual. Thus, when the certificate family is compatible and the network reaches this local regime, NCP can recover nearly optimal values for the original combinatorial problem.

\section{Experiments}
\paragraph{Instantiations of NCP.} We instantiate NCP on three representative CO problems that expose different missing structures. Details on specific NCP function for each class, network architectures, instance generation, training hyperparameters, and decoding rules can be found in Appendix~\ref{app:details}. Across all settings, the model predicts a perturbation object \(\Gamma\), the recovery layer computes \(\hat\mu_G=Q_G(\Gamma,y,\lambda)\), and training evaluates the unperturbed objective \(F_G(\hat\mu_G)\), so the perturbation acts only inside recovery and certificate consistency while all reported objectives use the original costs. In generalized assignment problems (GAP), NCP predicts item--machine perturbations, solves a capacity-dual recovery layer on perturbed costs, and uses the resulting reduced scores to form candidate assignments repaired by CP-SAT under the original costs and constraints. In maximum independent set problems (MIS), NCP predicts vertex perturbations, runs projected primal-dual recovery on an edge or clique-cover relaxation, and decodes an independent set from the recovered scores. In elementary shortest path problem with negative cycles (ESPP-NC), NCP predicts edge-cost perturbations, uses node-potential certificates to induce a directed acyclic recovery structure, and applies a discounted smoothed Bellman layer to produce transition probabilities and expected edge flows. We conduct out-of-distribution tests. For MIS, cross-dataset \textbf{zero-shot transfer results} are reported in Appendix~\ref{Appendix:Out}; for GAP and ESPP-NC, training and testing are conducted on different instance sizes and the results \textbf{are already indicative of out-of-distribution performances}.

\subsection{Main results}

\begin{table}[h!]
\scriptsize
\caption{In-distribution MIS on three benchmarks}
\label{tab:mis-main}
\centering
\setlength{\tabcolsep}{4.7pt}
\renewcommand{\arraystretch}{0.92}
\begin{tabular}{@{}l ccc ccc ccc@{}}
\toprule
& \multicolumn{3}{c}{TWITTER}
& \multicolumn{3}{c}{COLLAB}
& \multicolumn{3}{c}{IMDB} \\
\cmidrule(lr){2-4} \cmidrule(lr){5-7} \cmidrule(l){8-10}
Method & Ratio & Exact & Time
       & Ratio & Exact & Time
       & Ratio & Exact & Time \\
       &       &       & (ms)
       &       &       & (ms)
       &       &       & (ms) \\
\midrule
Gurobi        & -- & --      & 23.6
              & -- & --      & 12.7
              & -- & --      & 0.669 \\
KaMIS \cite{lamm2016finding}          & 1.0000 & 196/196   & 77.2
              & 1.0000 & 1000/1000 & 90.2
              & 1.0000 & 200/200   & 12.7 \\
\midrule
Erdos-Neural  & 0.9558 & 52/196    & 1.81
              & 0.9976 & 966/1000  & 0.910
              & 0.4373 & 30/200    & 0.110 \\
DIFUSCO       & 0.9809 & 92/196    & 303
              & 0.9980 & 970/1000  & 323
              & 0.9967 & 198/200   & 266 \\
Fast-T2T      & 0.9876 & 106/196   & 30.6
              & \textbf{0.9992} & 984/1000 & 26.7
              & 0.9978 & 198/200   & 22.7 \\
\midrule
NCP 
              & \textbf{0.9881} & \textbf{120/196} & 1.93
              & 0.9984 & \textbf{992/1000} & 0.970
              & \textbf{1.0000} & \textbf{200/200} & 0.140 \\
\bottomrule
\end{tabular}
\end{table}
\paragraph{MIS}
Tables~\ref{tab:mis-main} and~\ref{tab:wmis-main} summarize the results on unweighted and node-weighted MIS across TWITTER \cite{jure2014snap}, COLLAB \cite{yanardag2015deep, kersting2020benchmark}, and IMDB \cite{yanardag2015deep, kersting2020benchmark}. On the unweighted task, NCP is competitive with the strongest neural baselines while keeping the inference time close to single-pass methods. The advantage becomes clearer in the weighted setting. The performance of Erdos-Neural \cite{karalias2020erdos}, DIFUSCO \cite{sun2023difusco}, and Fast-T2T \cite{li2024fast} drops substantially on COLLAB and IMDB, whereas NCP remains above \(0.97\) on all three benchmarks and recovers the exact solution in most instances. These results suggest that NCP is not only efficient, but also more robust when the objective changes from counting selected nodes to optimizing weighted structure. 


\begin{table}[t]
\centering
\scriptsize
\caption{In-distribution vertex-weighted MIS on three benchmarks with
per-node weights $w\!\sim\!\mathcal U(0,1)$.}
\label{tab:wmis-main}
\setlength{\tabcolsep}{4.7pt}
\renewcommand{\arraystretch}{0.92}
\begin{tabular}{@{}l ccc ccc ccc@{}}
\toprule
& \multicolumn{3}{c}{TWITTER}
& \multicolumn{3}{c}{COLLAB}
& \multicolumn{3}{c}{IMDB} \\
\cmidrule(lr){2-4} \cmidrule(lr){5-7} \cmidrule(l){8-10}
Method & Ratio & Exact & Time
       & Ratio & Exact & Time
       & Ratio & Exact & Time \\
       &       &       & (ms)
       &       &       & (ms)
       &       &       & (ms) \\
\midrule
Gurobi
       & -- & -- & 32.2
       & -- & -- & 16.2
       & -- & -- & 0.65 \\
KaMIS 
       & 1.0000 & 196/196 & 112.6
       & 1.0000 & 1000/1000 & 92.7
       & 1.0000 & 200/200 & 14.4 \\
\midrule
Erdos-Neural
       & 0.8480 & 5/196 & 1.81
       & 0.6689 & 10/1000 & 0.910
       & 0.2656 & 2/200 & 0.110 \\
DIFUSCO
       & 0.8497 & 6/196 & 303
       & 0.6590 & 4/1000 & 323
       & 0.5928 & 2/200 & 266 \\
Fast-T2T
       & 0.8797 & 6/196 & 30.6
       & 0.7369 & 13/1000 & 26.7
       & 0.6661 & 4/200 & 22.7 \\
\midrule
NCP
       & \textbf{0.9795} & \textbf{39/196} & 1.93
       & \textbf{0.9768} & \textbf{941/1000} & 0.970
       & \textbf{0.9887} & \textbf{196/200} & 0.140 \\
\bottomrule
\end{tabular}
\end{table}

\paragraph{GAP}
Table~\ref{tab:gap-baselines} reports the mean gap to the ORLIB \cite{beasley1990or} optimum after 5-second CP-SAT repair on the top-\(k\) candidate mask. NCP improves over LP-round for all \(k\) and achieves the best mean gap for \(k=1\) and \(k=2\). At \(k=3\), it remains close to Fast-T2T while reducing runtime from \(28.420\) seconds to \(0.123\) seconds. With \(k=2\), NCP solves \(32/60\) instances exactly, compared with \(17/60\) for Fast-T2T and \(13/60\) for DIFUSCO. These results suggest that \(\Gamma\) produces a compact candidate set whose LP solution is already close to the integer optimum.


\begin{table}[h!]
\centering
\scriptsize
\setlength{\tabcolsep}{11pt}
\renewcommand{\arraystretch}{0.90}
\caption{
\textbf{Mean gap (\%)} is the relative gap to the ORLIB optimum after CP-SAT
repair with a 5\,s budget on the candidate mask.}
\label{tab:gap-baselines}
\begin{tabular}{@{}l ccc ccc c@{}}
\toprule
\multirow{2}{*}{Method}
& \multicolumn{3}{c}{Mean Gap (\%)}
& \multicolumn{3}{c}{Exact}
& \multirow{2}{*}{Time\ (s)} \\
\cmidrule(lr){2-4}\cmidrule(lr){5-7}
& $k=1$ & $k=2$ & $k=3$
& $k=1$ & $k=2$ & $k=3$
& \\
\midrule
LP-round
& 1.2037 & 0.4840 & 0.2716
& 0/60 & 5/60 & 13/60
& 0.056 \\
Erdos-Neural
& 3.5931 & 2.6334 & 2.0377
& 0/60 & 0/60 & 1/60
& 0.042 \\
DIFUSCO
& 5.2358 & 0.5609 & 0.0361
& 0/60 & 13/60 & 47/60
& 12.142 \\
Fast-T2T
& 3.7490 & 0.3974 & \textbf{0.0153}
& 0/60 & 17/60 & 54/60
& 28.420 \\
\midrule
NCP
& \textbf{1.0365} & \textbf{0.1195} & 0.0172
& 0/60 & \textbf{32/60} & \textbf{52/60}
& 0.123 \\
\bottomrule
\end{tabular}
\end{table}

\paragraph{ESPP-NC} 
Table~\ref{tab:er-ba-transfer-sweep} reports best path costs on ER~\cite{erdds1959random, gilbert1959random} and BA~\cite{barabasi1999emergence} graphs with \(n\in\{50,100,200,500\}\), where lower values are better. Exact labeling gives the strongest result at the smallest size but does not scale within the time budget. Among scalable methods, NCP remains close to Label-1M at \(n=50\) and consistently outperforms LP+Dec. on larger graphs. The advantage becomes more pronounced as the graph size increases. This suggests that the learned certificate structure remains effective beyond the small instances where exact labeling is available, and that NCP provides a more scalable mechanism than LP-based decoding.

\begin{table*}[t]
\setlength{\tabcolsep}{4pt}
\renewcommand{\arraystretch}{0.90}
\centering
\scriptsize
\caption{Performance and runtime comparison on ER and BA graphs (time in parentheses).}
\label{tab:er-ba-transfer-sweep}
\setlength{\tabcolsep}{4.8pt}
\renewcommand{\arraystretch}{0.92}
\begin{tabular}{lcccccccc}
\toprule
& \multicolumn{2}{c}{$n=50$} & \multicolumn{2}{c}{$n=100$} & \multicolumn{2}{c}{$n=200$} & \multicolumn{2}{c}{$n=500$} \\
\cmidrule(lr){2-3}\cmidrule(lr){4-5}\cmidrule(lr){6-7}\cmidrule(lr){8-9}
Method & ER & BA & ER & BA & ER & BA & ER & BA \\
\midrule
Label-1M & \makecell{\textbf{-32.59} \\ (59.5s)} & \makecell{-17.22 \\ (33.9s)} & -- & -- & -- & -- & -- & -- \\
Beam-256 & \makecell{-30.65 \\ (30.02s)} & \makecell{\textbf{-17.79} \\ (30.02s)} & \makecell{-56.38 \\ (30.02)} & \makecell{-40.79 \\ (30.2s)} & \makecell{-103.14 \\ (30.01s)} & \makecell{-85.87 \\ (30.04s)} & \makecell{-161.13 \\ (30.05s)} & \makecell{-79.69 \\ (30.01s)} \\
Beam-512 & \makecell{-31.32 \\ (30.02s)} & \makecell{-16.50 \\ (30.01s)} & \makecell{-57.35 \\ (30.01)} & \makecell{-40.78 \\ (30.02s)} & \makecell{-103.11 \\ (1.572s)} & \makecell{-67.68 \\ (30.03s)} & \makecell{-138.05 \\ (30.04s)} & \makecell{-63.69 \\ (30.08s)} \\
\midrule
LP+Dec. & \makecell{-24.15 \\ (0.115s)} & \makecell{-14.60 \\ (0.060s)} & \makecell{-50.78 \\ (0.312s)} & \makecell{-37.51 \\ (0.175s)} & \makecell{-102.57 \\ (0.920s)} & \makecell{-85.13 \\ (0.518s)} & \makecell{-259.30 \\ (4.89s)} & \makecell{-236.32 \\ (2.68s)} \\
Erdos & \makecell{-2.20 \\ (0.010s)} & \makecell{-1.60 \\ (0.009s)} & \makecell{-2.82 \\ (0.016s)} & \makecell{-2.28 \\ (0.012s)} & \makecell{-4.03 \\ (0.029s)} & \makecell{-2.68 \\ (0.021s)} & \makecell{-6.59 \\ (0.110s)} & \makecell{-4.67 \\ (0.054s)} \\
\midrule
NCP & \makecell{-31.22 \\ (0.129s)} & \makecell{-16.41 \\ (0.069s)} & \makecell{\textbf{-71.21} \\ (0.316s)} & \makecell{\textbf{-47.97} \\ (0.193s)} & \makecell{\textbf{-151.17} \\ (0.818s)} & \makecell{\textbf{-120.16} \\ (0.525s)} & \makecell{\textbf{-395.11} \\ (3.40s)} & \makecell{\textbf{-349.10} \\ (2.05s)} \\
\bottomrule
\end{tabular}
\end{table*}

\subsection{Validating the structural properties}~\label{subsec:ncl-validation}
\noindent\textbf{Component ablation.}
Table~\ref{tab:ncl-ablation-combined} ablates the main components of NCP on weighted MIS. Removing the structure certificate causes the largest degradation on COLLAB and IMDB, showing that the certificate carries the main structural signal. Randomizing \(\Gamma\) also substantially weakens performance, suggesting that the certificate benefits from a learned perturbation rather than an arbitrary one. The degradations from removing certificate consistency and decoding directly from \(\Gamma\) further show that the primal signal is recovered through \(\hat\mu = Q_G(\Gamma,y,\lambda)\), not from \(\Gamma\) alone.

\begin{table}[t]
  \centering
  \setlength{\tabcolsep}{3.5pt}
    \scriptsize
  \caption{NCP component ablation on weighted-MIS. Full model reports the original performance, while ablated variants report degradation relative to the full model. Larger degradation indicates stronger performance loss.}
  \label{tab:ncl-ablation-combined}
  \begin{tabular}{lcccccc}
    \toprule
    Variant 
    & \multicolumn{2}{c}{Twitter} 
    & \multicolumn{2}{c}{COLLAB} 
    & \multicolumn{2}{c}{IMDB} \\
    \cmidrule(lr){2-3}\cmidrule(lr){4-5}\cmidrule(lr){6-7}
    & Ratio & Exact 
    & Ratio & Exact 
    & Ratio & Exact \\
    \midrule
    NCP (Full model) 
    & 0.9795 & 39/196 
    & 0.9768 & 941/1000 
    & 0.9887 & 196/200 \\
    \midrule
    w/o certificate 
    & $\downarrow$0.1315 & $\downarrow$17.3 pp (34) 
    & $\downarrow$0.3079 & $\downarrow$93.1 pp (931) 
    & $\downarrow$0.7231 & $\downarrow$97.0 pp (194) \\

    w/o perturb 
    & $\downarrow$0.0403 & $\downarrow$11.7 pp (23) 
    & $\downarrow$0.0195 & $\downarrow$24.8 pp (248) 
    & $\downarrow$0.0720 & $\downarrow$27.5 pp (55) \\

    Random perturb 
    & $\downarrow$0.0978 & $\downarrow$18.4 pp (36) 
    & $\downarrow$0.1580 & $\downarrow$91.2 pp (912) 
    & $\downarrow$0.2203 & $\downarrow$91.5 pp (183) \\

    w/o consistency 
    & $\downarrow$0.0422 & $\downarrow$12.2 pp (24) 
    & $\downarrow$0.0935 & $\downarrow$65.6 pp (656) 
    & $\downarrow$0.0772 & $\downarrow$35.0 pp (70) \\

    Decoder by perturb 
    & $\downarrow$0.2487 & $\downarrow$19.4 pp (38) 
    & $\downarrow$0.3145 & $\downarrow$93.4 pp (934) 
    & $\downarrow$0.4086 & $\downarrow$95.0 pp (190) \\
    \bottomrule
  \end{tabular}
\end{table}

\begin{figure}
    \centering
    \resizebox{\linewidth}{!}{%
        \includegraphics[height=0.5cm]{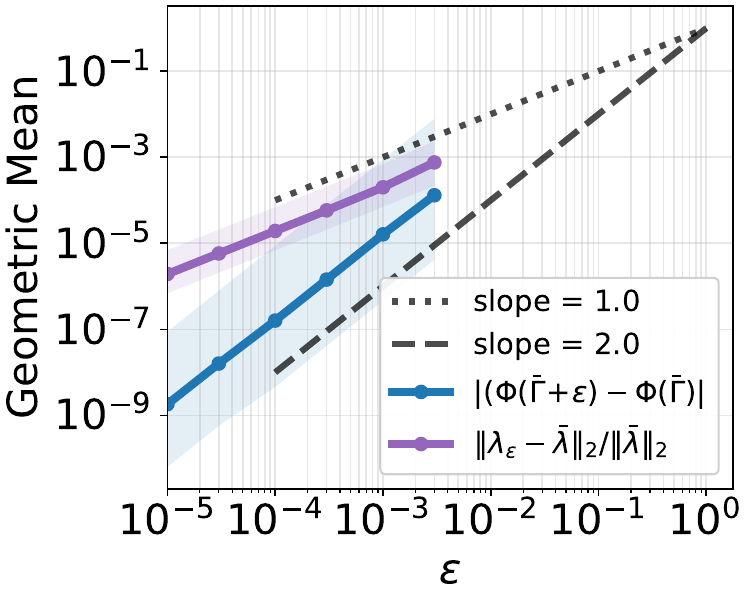}\hspace{-0.05cm}
        \includegraphics[height=0.5cm]{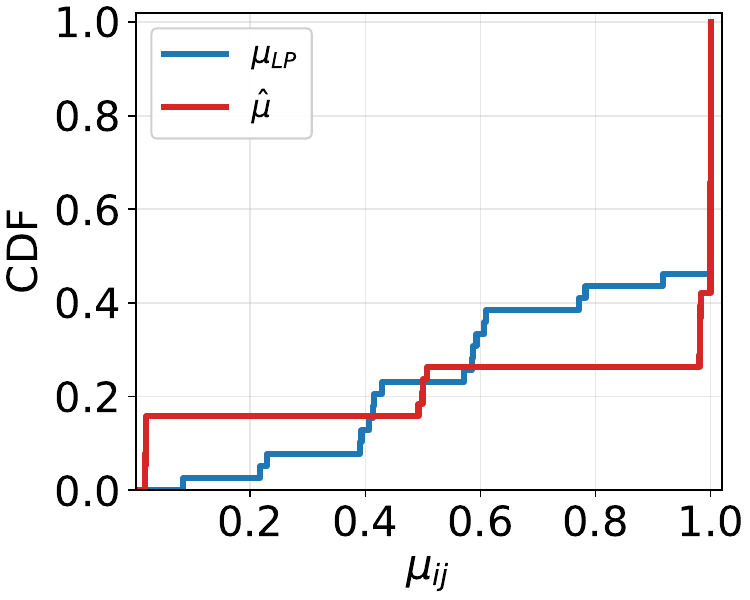}\hspace{-0.05cm}
        \includegraphics[height=0.5cm]{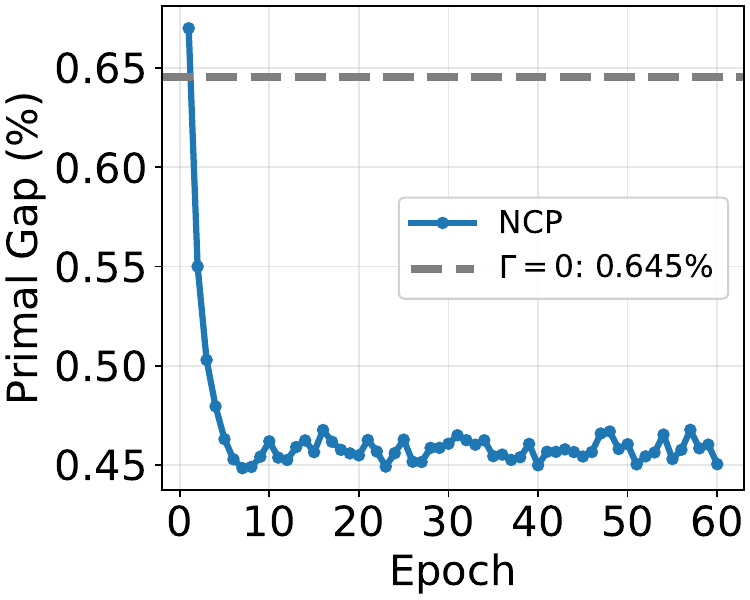}
    }
        \caption{Structural diagnostics on GAP.
        \textbf{Left:} test-set verification of
        Theorem~\ref{thm:second-order-stability}. Around the trained
        \(\bar\Gamma=\Gamma_\theta(G)\), additive perturbations of magnitude
        \(\varepsilon\) induce \(O(\varepsilon)\) drift in the certificate
        state \(\lambda\) and \(O(\varepsilon^2)\) drift in the reduced
        objective \(\Phi\), matching the predicted slopes of \(1\) and \(2\)
        over four orders of magnitude in \(\varepsilon\).
        \textbf{Middle:} non-zero marginal CDFs on instance \texttt{c1060-1}: the
        LP-relaxation marginal \(\mu_{\mathrm{LP}}\) is spread across the
        fractional interior, while the NCP-recovered marginal \(\hat\mu\)
        concentrates near \(\{0,1\}\), evidencing the integerizing role of
        the learned \(\Gamma\) anticipated by
        Proposition~\ref{prop:linear_sufficiency}.
        \textbf{Right:} primal-gap evolution on the test set; at every epoch
        we run the certificate fixed-point and recover \(\hat\mu\). NCP
        drops below Lagrangian bound (\(\Gamma=0\)) within five epochs and stabilizes near \(0.46\%\).}
    \label{fig:gap_results}
\end{figure}

\paragraph{Learning shifts the primal bound below the Lagrangian.}
Figure~\ref{fig:gap_results} (\textbf{Right}) traces the primal gap on GAP over training. The dashed line at \(\Gamma=0\) is the canonical Lagrangian relaxation  with the same recovery layer. NCP drops below this bound within five epochs and stabilizes near \(0.46\%\), supplying the residual price action that the dual prices alone cannot represent, which is exactly the role assigned to \(\Gamma\) by Proposition~\ref{prop:linear_sufficiency}.

\paragraph{Second-order value stability.}
Theorem~\ref{thm:second-order-stability} asserts that, near a stationary point \(\bar\Gamma\), an \(O(\varepsilon)\) perturbation of the input induces \(O(\varepsilon)\) drift in the certificate state and \(O(\varepsilon^2)\) drift in the reduced objective. Figure~\ref{fig:gap_results} (\textbf{Left}) reports the geometric mean of \(|\Phi(\bar\Gamma+\varepsilon)-\Phi(\bar\Gamma)|\) and \(\|\lambda_\varepsilon-\bar\lambda\|_2/\|\bar\lambda\|_2\) as \(\varepsilon\) varies. The observed slopes are close to \(2\) for the reduced objective and close to \(1\) for the certificate state, matching the predicted local scaling over the tested range of \(\varepsilon\). The reduced objective therefore behaves as a smooth surface in \(\Gamma\) within the regular regime, and gradient descent on \(\mathcal L_{\mathrm{NCP}}\) can shrink the \(O(\varepsilon^2)\) residual in Corollary~\ref{cor:objective-exact-second-order} without destabilizing the certificate state.

\bibliography{reference}

\clearpage
\appendix

\section{Derivation of the Variational Outer Objective}
\label{app:kl_derivation}

For $\tau>0$, define the Gibbs distribution over feasible configurations $x\in\mathcal X_G$ by
\[
    P_\tau(x)
    =
    \frac{\exp(-c^\top\phi(x)/\tau)}{Z_\tau},
    \qquad
    Z_\tau
    =
    \sum_{x\in\mathcal X(G)}
    \exp(-c^\top\phi(x)/\tau).
\]
For any distribution $q$ over $\mathcal X_G$, let
\[
    \mu_q=\mathbb E_q[\phi(x)].
\]
Then
\[
\begin{aligned}
    \mathrm{KL}(q\|P_\tau)
    &=
    \sum_x q(x)\log\frac{q(x)}{P_\tau(x)}
    \\
    &=
    -H(q)
    +
    \frac{1}{\tau}\mathbb E_q[c^\top\phi(x)]
    +
    \log Z_\tau
    \\
    &=
    -H(q)
    +
    \frac{1}{\tau}\langle c,\mu_q\rangle
    +
    \log Z_\tau .
\end{aligned}
\]
For a fixed mean parameter $\mu$, minimizing this expression over all
distributions $q$ satisfying $\mathbb E_q[\phi]=\mu$ is equivalent to maximizing its entropy. By definition,
\[
    H_{\mathcal M}(\mu)
    =
    \sup_{q:\mathbb E_q[\phi]=\mu} H(q).
\]
Therefore,
\[
\begin{aligned}
    \inf_q \mathrm{KL}(q\|P_\tau)
    &=
    \inf_{\mu\in\mathcal M_G}
    \left\{
        \frac{1}{\tau}\langle c,\mu\rangle
        -
        H_{\mathcal M}(\mu)
    \right\}
    +
    \log Z_\tau
    \\
    &=
    \frac{1}{\tau}
    \inf_{\mu\in\mathcal M_G}
    \left\{
        \langle c,\mu\rangle
        -
        \tau H_{\mathcal M}(\mu)
    \right\}
    +
    \log Z_\tau .
\end{aligned}
\]
Thus, for $\tau>0$, minimizing
\[
    F_G(\mu)
    =
    \langle c,\mu\rangle-\tau H_{\mathcal M}(\mu)
\]
over the exact marginal polytope is equivalent, up to the positive scaling factor $1/\tau$ and the additive constant $\log Z_\tau$, to minimizing $\mathrm{KL}(q\|P_\tau)$ over distributions on feasible configurations. When $\tau=0$, the same objective reduces to the original linear combinatorial objective. 

\section{Proofs and additional materials for Section~\ref{sec:framework}}

\subsection{Proof of Proposition~\ref{prop:linear_sufficiency}}
\label{app:proof_linear_sufficiency}
\noindent\textbf{Proposition~\ref{prop:linear_sufficiency} 
(Sufficiency of linear cost perturbation).}
\textit{Fix a certificate $y$ such that an optimizer
$\mu^*$ belongs to $S_G(y)$.
Suppose the valid inequalities omitted from the recovery layer in certifying
$\mu^*$ can be written as $B_G\mu\le d_G$, and let $(\lambda^*,z^*)$ be KKT
multipliers at $\mu^*$ for minimizing $F_G$ over $S_G(y)$ subject to
$A_G\mu\le b_G$ and $B_G\mu\le d_G$. Then,
\[
\Gamma^*:=B_G^\top z^*
\quad\Longrightarrow\quad
\mu^*\in Q_G(\Gamma^*,y,\lambda^*).
\]}


\begin{proof}
Consider the certificate-restricted convex problem
\[
\min_{\mu\in S_G(y)} F_G(\mu)
\quad \mathrm{s.t.}\quad
A_G\mu\le b_G,\qquad B_G\mu\le d_G .
\]
By the assumption that the omitted valid inequalities complete the recovery
description around the certificate $y$, $\mu^*$ is an optimizer of this problem.
Since $(\lambda^*,z^*)$ are KKT multipliers at $\mu^*$ for this problem, there
exists a subgradient $\xi^*\in\partial F_G(\mu^*)$ such that
\[
0\in \xi^*+A_G^\top\lambda^*+B_G^\top z^*+N_{S_G(y)}(\mu^*).
\]
Equivalently, for every \(\mu\in S_G(y)\),
\[
\left\langle
\xi^*+A_G^\top\lambda^*+B_G^\top z^*,
\mu-\mu^*
\right\rangle\ge 0.
\]
Define
\[
\Gamma^*:=B_G^\top z^* .
\]
Then, for every \(\mu\in S_G(y)\),
\[
\left\langle
\xi^*+\Gamma^*+A_G^\top\lambda^*,
\mu-\mu^*
\right\rangle\ge 0.
\]
This is the first-order optimality condition for the convex recovery problem
\[
\min_{\mu\in S_G(y)}
\left\{
F_G(\mu)+
\left\langle \Gamma^*+A_G^\top\lambda^*,\mu\right\rangle
\right\}.
\]
Therefore,
\[
\mu^*\in Q_G(\Gamma^*,y,\lambda^*).
\]
\end{proof}

\subsection{Proof of Proposition~\ref{prop:consistent_recovery}}
\label{app:proof_consistent_recovery}
\noindent\textbf{Proposition~\ref{prop:consistent_recovery} (Feasibility and shifted optimality of consistent recovery)} 
\textit{Let $\hat\mu=Q_G(\Gamma,y,\lambda)$ and suppose $(y,\lambda)$ is a consistent
certificate at perturbation $\Gamma$. Then 
\[
\hat\mu\in\mathcal M_G.
\]
Moreover, for every $\nu\in\mathcal S_G(y)\cap\mathcal P_G$,
\begin{equation}
    F_G(\hat\mu)-F_G(\nu)
    \le
    \langle \Gamma,\nu-\hat\mu\rangle
\end{equation}}

\begin{proof}
Because $\hat\mu=Q_G(\Gamma,y,\lambda)$, we have
$\hat\mu\in\mathcal S_G(y)$. Since $(y,\lambda)$ is a consistent
certificate at perturbation $\Gamma$, we also have
\[
    A_G\hat\mu\le b_G .
\]
Together with $\mathcal S_G(y)\subseteq\mathcal P_{\mathrm{loc}}(G)$,
this implies
\[
    \hat\mu\in\mathcal S_G(y)\cap\mathcal P_G .
\]
By the structure-preserving property
\[
    \mathcal S_G(y)\cap\mathcal P_G\subseteq\mathcal M_G,
\]
we obtain
\[
    \hat\mu\in\mathcal M_G .
\]

It remains to prove the shifted optimality inequality. Since
$\hat\mu \in Q_G(\Gamma,y,\lambda)$, for every
$\nu\in\mathcal S_G(y)$,
\[
    F_G(\hat\mu)
    +
    \left\langle
        \Gamma+A_G^\top\lambda,\hat\mu
    \right\rangle
    \le
    F_G(\nu)
    +
    \left\langle
        \Gamma+A_G^\top\lambda,\nu
    \right\rangle .
\]
Rearranging gives
\[
    F_G(\hat\mu)-F_G(\nu)
    \le
    \langle \Gamma,\nu-\hat\mu\rangle
    +
    \lambda^\top A_G(\nu-\hat\mu).
\]
Now take any $\nu\in\mathcal S_G(y)\cap\mathcal P_G$. Then
$A_G\nu\le b_G$. Since consistency also gives
$\lambda\ge0$ and
$\lambda^\top(A_G\hat\mu-b_G)=0$, we have
\[
    \lambda^\top A_G(\nu-\hat\mu)
    =
    \lambda^\top(A_G\nu-b_G)
    -
    \lambda^\top(A_G\hat\mu-b_G)
    \le 0 .
\]
Therefore
\[
    F_G(\hat\mu)-F_G(\nu)
    \le
    \langle \Gamma,\nu-\hat\mu\rangle .
\]
This proves the claim.
\end{proof}

\subsection{Implicit differentiation through certificate consistency}
\label{app:implicit_gradient}

Let \(\mathcal T^\alpha_{\Gamma,G}\) denote the damped update in
\eqref{eq:NCP_forward_iteration}, and define the fixed-point residual
\[
    R_{\Gamma,G}(z)
    =
    z-\mathcal T^\alpha_{\Gamma,G}(z).
\]
At a fixed point \(z^\star\), implicit differentiation is well defined under
standard nondegeneracy conditions. We require the structured recovery solution
to be locally unique and differentiable in \((\Gamma,y,\lambda)\), the projection
update to have a locally fixed active set, and
\(\partial_z R_{\Gamma_\theta(G),G}(z^\star)\) to be nonsingular
~\citep{dontchev2009implicit}. In our setting, entropy regularization gives a
unique smooth recovery, the projection is piecewise affine under a fixed active
set, and local stability of the damped update gives the required nonsingularity.

Under these conditions, the implicit function theorem
~\citep{krantz2002implicit} gives
\begin{equation}
    \frac{d z^\star}{d\theta}
    =
    -\left[\partial_z R_{\Gamma_\theta(G),G}(z^\star)\right]^{-1}
    \partial_\Gamma R_{\Gamma_\theta(G),G}(z^\star)
    \frac{\partial\Gamma_\theta(G)}{\partial\theta}.
    \label{eq:NCP_implicit_gradient}
\end{equation}
The gradient of \(\mathcal L_{\mathrm{NCP}}\) follows by differentiating
\[
    \mu_\theta(G)
    =
    Q_G(\Gamma_\theta(G),y^\star,\lambda^\star)
\]
together with \eqref{eq:NCP_implicit_gradient}. This avoids backpropagation
through long unrolled fixed-point iterations. Finite unrolling can also be used
as a direct computational surrogate, as in standard differentiable optimization
and equilibrium layers
~\citep{amos2017optnet,agrawal2019differentiable,bai2019deep}.

\section{Proofs for Section~\ref{sec:theory}}
\label{app:optimization-proofs}

\subsection{Proof of Theorem~\ref{thm:strong-stationarity}}
\noindent\textbf{Theorem~\ref{thm:strong-stationarity} (Reduced stationarity of OC)}
\textit{At a regular OC solution, the implicit reduction defines a \(C^2\)
objective \(\Phi_G\) on \(U\). Moreover,
\[
    \nabla_\Gamma\Phi_G(\bar\Gamma)=0
\]
if and only if \((\bar z,\bar\Gamma)\) satisfies the KKT stationarity
conditions of the local equality-constrained OC problem
\[
    \min_{z,\Gamma}
    F_G(\hat\mu_G(z,\Gamma))
    \quad
    \mathrm{s.t.}
    \quad
    \mathcal K_G(z,\Gamma)=0 .
\]
}

\begin{proof}
Since \((\bar z,\bar\Gamma)\) is a regular OC point,
\[
    \mathcal K_G(\bar z,\bar\Gamma)=0
    \qquad\text{and}\qquad
    \partial_z\mathcal K_G(\bar z,\bar\Gamma)
    \text{ is nonsingular}.
\]
By the implicit function theorem applied to
\[
    \mathcal K_G(z,\Gamma)=0,
\]
there exist neighborhoods of \(\bar z\) and \(\bar\Gamma\), and a locally
unique \(C^2\) solution map
\[
    \Gamma\mapsto z_\Gamma
\]
such that
\[
    \mathcal K_G(z_\Gamma,\Gamma)=0,
    \qquad
    z_{\bar\Gamma}=\bar z .
\]
Differentiating \(\mathcal K_G(z_\Gamma,\Gamma)=0\) with respect to
\(\Gamma\) gives
\begin{equation}
    \frac{d z_\Gamma}{d\Gamma}
    =
    -
    \bigl(\partial_z\mathcal K_G(z_\Gamma,\Gamma)\bigr)^{-1}
    \partial_\Gamma\mathcal K_G(z_\Gamma,\Gamma).
    \label{eq:proof-implicit-sensitivity}
\end{equation}

The reduced objective is
\[
    \Phi_G(\Gamma)
    =
    F_G\bigl(\hat\mu_G(z_\Gamma,\Gamma)\bigr).
\]
Applying the chain rule and substituting
\eqref{eq:proof-implicit-sensitivity} yields
\begin{equation}
\begin{aligned}
    \nabla_\Gamma\Phi_G(\Gamma)
    &=
    \nabla_\Gamma
    \bigl[F_G\circ\hat\mu_G\bigr](z_\Gamma,\Gamma)
    \\
    &\quad
    -
    \bigl(\partial_\Gamma\mathcal K_G(z_\Gamma,\Gamma)\bigr)^\top
    \bigl(\partial_z\mathcal K_G(z_\Gamma,\Gamma)\bigr)^{-\top}
    \nabla_z
    \bigl[F_G\circ\hat\mu_G\bigr](z_\Gamma,\Gamma).
\end{aligned}
\label{eq:proof-reduced-gradient}
\end{equation}
This is the implicit-gradient formula for the reduced objective.

It remains to relate reduced stationarity to KKT stationarity of the local
equality-constrained OC problem
\[
    \min_{z,\Gamma}
    \quad
    F_G\bigl(\hat\mu_G(z,\Gamma)\bigr)
    \qquad
    \mathrm{s.t.}
    \qquad
    \mathcal K_G(z,\Gamma)=0 .
\]
The KKT stationarity equations are
\[
    \nabla_z
    \bigl[F_G\circ\hat\mu_G\bigr](z,\Gamma)
    +
    \bigl(\partial_z\mathcal K_G(z,\Gamma)\bigr)^\top \xi
    =
    0,
\]
and
\[
    \nabla_\Gamma
    \bigl[F_G\circ\hat\mu_G\bigr](z,\Gamma)
    +
    \bigl(\partial_\Gamma\mathcal K_G(z,\Gamma)\bigr)^\top \xi
    =
    0.
\]
Since \(\partial_z\mathcal K_G(z,\Gamma)\) is nonsingular locally, the first
stationarity equation uniquely determines
\[
    \xi^\star
    =
    -
    \bigl(\partial_z\mathcal K_G(z,\Gamma)\bigr)^{-\top}
    \nabla_z
    \bigl[F_G\circ\hat\mu_G\bigr](z,\Gamma).
\]
Substituting this multiplier into the second stationarity equation gives
exactly
\[
    \nabla_\Gamma\Phi_G(\Gamma)=0
\]
by \eqref{eq:proof-reduced-gradient}. Therefore reduced stationarity is
equivalent to KKT stationarity of the local equality-constrained OC problem.
Evaluating at \((\bar z,\bar\Gamma)\) proves the theorem.
\end{proof}
\subsection{Proof of Theorem~\ref{thm:second-order-stability}}

\noindent\textbf{Theorem~\ref{thm:second-order-stability} (Second-order value stability)}
\textit{Suppose the point \((\bar z,\bar\Gamma)\) is a regular OC solution and \(\nabla_\Gamma\Phi_G(\bar\Gamma)=0\). For any
\(\Gamma_\varepsilon\in U\) with \(\Gamma_\varepsilon-\bar\Gamma=O(\varepsilon)\),
\[
    z_{\Gamma_\varepsilon}-\bar z=O(\varepsilon),
    \qquad
    \hat\mu_G(z_{\Gamma_\varepsilon},\Gamma_\varepsilon)
    -
    \hat\mu_G(\bar z,\bar\Gamma)
    =
    O(\varepsilon),
\]
and
\[
    \Phi_G(\Gamma_\varepsilon)-\Phi_G(\bar\Gamma)
    =
    O(\varepsilon^2).
\]
}

\begin{proof}
By Theorem~\ref{thm:strong-stationarity}, the local solution map
\[
    \Gamma\mapsto z_\Gamma
\]
is \(C^2\). Hence it is locally Lipschitz. If
\[
    \Gamma_\varepsilon-\bar\Gamma=O(\varepsilon),
\]
then
\[
    z_{\Gamma_\varepsilon}-z_{\bar\Gamma}=O(\varepsilon).
\]
Since \(z_{\bar\Gamma}=\bar z\), this gives
\[
    z_{\Gamma_\varepsilon}-\bar z=O(\varepsilon).
\]
Because \(\hat\mu_G\) is \(C^2\) locally,
\[
    \hat\mu_G(z_{\Gamma_\varepsilon},\Gamma_\varepsilon)
    -
    \hat\mu_G(\bar z,\bar\Gamma)
    =
    O(\varepsilon).
\]

The reduced objective \(\Phi_G\) is \(C^2\). Taylor expansion at
\(\bar\Gamma\) gives
\[
    \Phi_G(\Gamma_\varepsilon)
    =
    \Phi_G(\bar\Gamma)
    +
    \nabla_\Gamma\Phi_G(\bar\Gamma)^\top
    (\Gamma_\varepsilon-\bar\Gamma)
    +
    O(\|\Gamma_\varepsilon-\bar\Gamma\|^2).
\]
Since
\[
    \nabla_\Gamma\Phi_G(\bar\Gamma)=0
\]
and
\[
    \Gamma_\varepsilon-\bar\Gamma=O(\varepsilon),
\]
we obtain
\[
    \Phi_G(\Gamma_\varepsilon)-\Phi_G(\bar\Gamma)
    =
    O(\varepsilon^2).
\]
\end{proof}

\subsection{Proof of Corollary~\ref{cor:objective-exact-second-order}}

\noindent\textbf{Corollary~\ref{cor:objective-exact-second-order} (Suboptimality decomposition)}
\textit{Let $(\bar z,\bar\Gamma)$ be a regular OC solution such that
$\bar\Gamma$ is a stationary point of $\Phi_G$ and $\bar z=z_{\bar\Gamma}$,
and suppose the structure-preservation condition~\eqref{eq:structure_preserving} holds on $U$. Define
\[
\Delta_{\mathrm{cert}}(G,\bar\Gamma)
:=
\Phi_G(\bar\Gamma)-p^*(G),
\]
where $p^*(G)$ is the optimal value of the original unperturbed marginal
problem. Then, for any $\Gamma_\epsilon\in U$ with
$\Gamma_\epsilon-\bar\Gamma=O(\epsilon)$,
\[
\Phi_G(\Gamma_\epsilon)-p^*(G)
=
\Delta_{\mathrm{cert}}(G,\bar\Gamma)+O(\epsilon^2).
\]
If, in addition, $\bar\Gamma$ satisfies the sufficient condition of
Proposition~\ref{prop:linear_sufficiency}, then $\Delta_{\mathrm{cert}}(G,\bar\Gamma)=0$, and
\[
0\le \Phi_G(\Gamma_\epsilon)-p^*(G)=O(\epsilon^2).
\]
}

\begin{proof}
The decomposition
\[
    \Phi_G(\Gamma_\varepsilon)-p^\star(G)
    =
    \bigl[\Phi_G(\bar\Gamma)-p^\star(G)\bigr]
    +
    \bigl[\Phi_G(\Gamma_\varepsilon)-\Phi_G(\bar\Gamma)\bigr]
    =
    \Delta_{\mathrm{cert}}(G,\bar\Gamma)+O(\varepsilon^2)
\]
follows from the definition of \(\Delta_{\mathrm{cert}}\) and
Theorem~\ref{thm:second-order-stability}.

Under Proposition~\ref{prop:linear_sufficiency},
\(\hat\mu_G(\bar z,\bar\Gamma)\) is optimal for the original marginal
problem, so \(\Phi_G(\bar\Gamma)=p^\star(G)\) and
\(\Delta_{\mathrm{cert}}(G,\bar\Gamma)=0\). 
Structure preservation~\eqref{eq:structure_preserving}, together with certificate consistency on $U$,
ensures that the recovered marginal at $\Gamma_\epsilon$ lies in
$\mathcal M_G$, so $\Phi_G(\Gamma_\epsilon)\ge p^*(G)$, hence the nonnegative \(O(\varepsilon^2)\) bound.
\end{proof}

\section{Implementation Details}
\label{app:details}

This appendix gives, for each of the three benchmarks in the experiments, the dataset construction, the network used to predict $\Gamma_\theta(G)$, the inner recovery layer $Q_G$, the certificate $y$ and its readout $C_G$, the training procedure, and the implementation of the baselines we compare against. Throughout, $\Gamma$, $y$, $\lambda$, $Q_G$, $C_G$, and $\mathcal S_G(y)$ refer to the objects defined in Section~\ref{sec:framework}.

\subsection{Experimental environment}
Experiments were conducted on a workstation equipped with an AMD Ryzen Threadripper PRO 7955WX v16 CPU and 1 NVIDIA RTX 5090, running Ubuntu \texttt{22.04} LTS. The models were implemented using Python \texttt{3.10} and PyTorch \texttt{2.7.0}, with additional libraries including torch-scatter \texttt{2.1.2} and torch-geometric \texttt{2.6.1}.
For reproducibility, a random seed of \texttt{42} was set for all random number generators. The code for this experiment is available at~\url{https://anonymous.4open.science/r/Neural-Certificate-Pricing-D515/README.md}.

\subsection{Generalized Assignment Problem (GAP)}
\label{app:details:gap}

\paragraph{Instances.}
We use the ORLIB GAP instances~\cite{beasley1990or}, partitioned into a $60\%/20\%/20\%$ train/validation/test split that is shared across all methods. Following the strict no-reject convention, profit maximization is restated as cost minimization without dummy reject rows, and both the inner recovery layer and the final repair use the original ORLIB cost vector $c$ and capacity vector.

\paragraph{Outer network for $\Gamma_\theta(G)$.}
$\Gamma_\theta(G)$ is an item--machine cost shift produced by a bipartite heterogeneous attention encoder over the resource--job graph, with resource, job, and edge features capturing capacity slack, cost ranks, and warm-start usage. A gauge head decomposes the raw edge score into a usage-aligned component and an orthogonal component, so that $\Gamma_\theta(G)$ acts on reduced costs rather than on the cost vector directly; the resulting perturbation is bounded through a smooth clip.

\paragraph{Inner recovery layer.}
The inner layer $Q_G(\Gamma,y,\lambda)$ is the differentiable assignment LP on the perturbed cost $\tilde c=c+\Gamma_\theta(G)$ over the local assignment polytope with the original capacity constraints using \texttt{cvxpylayers}~\cite{agrawal2019differentiable}. The
capacity dual $\lambda$ is read off in parallel via a non-differentiable LP solve; because $\lambda$ is piecewise constant in $c$, the gradient through the recovered marginal $\hat\mu$ alone is correct.

\paragraph{Certificate and decoder.}
The certificate $y$ is a per-job top-$k$ candidate mask induced by the
reduced score $\hat r_{ij}=-(c_{ij}+\gamma\lambda_i u_{ij})$ with
$k\le3$, and the warm-start assignment is unioned in to guarantee
feasibility for all methods. The readout $C_G(\hat\mu)$ applies the same top-$k$ rule
to the reduced score derived from $\hat\mu=Q_G(\Gamma,y,\lambda)$.
The certificate-restricted recovery decodes an exact assignment over
the candidate edges under the original cost and capacity, with a
$5$-second per-instance time limit.

\paragraph{Training.}
We jointly train the outer encoder, the gauge head, and the steering head under
\begin{equation*}
    \mathcal L
    \;=\; \langle c,\hat\mu\rangle
    \;+\; w_\delta\,\|\Gamma_\theta(G)\|^{2}
    \;+\; w_{\text{st}}\,\|q\|^{2},
\end{equation*}
where $\hat\mu$ is the recovered soft assignment and $q$ is the gauge-head projection magnitude. Optimization uses AdamW with learning rate $3\!\times\!10^{-4}$, $60$ epochs, and $w_\delta=w_{\text{st}}=10^{-2}$. Checkpoints are selected by the mean
optimality gap on the validation split.

\paragraph{Baselines.}
\begin{description}

\item[\textsc{LP-round}.] Solve the unperturbed assignment LP, round the primal solution per job to its \texttt{argmax}, and pass the resulting mask through the same exact-assignment repair.
\item[\textsc{DIFUSCO}, \textsc{Fast-T2T}.] Both adapt the diffusion decoders of \cite{sun2023difusco,li2024fast} to GAP. The
encoder is a $12$-layer gated GCN; supervision is a per-edge $\{0,1\}$ label obtained by solving each training instance to optimality. \textsc{DIFUSCO} trains a categorical discrete diffusion model; \textsc{Fast-T2T} trains the same encoder under a consistency loss and decodes in one step plus a single rewrite. Both produce an edge-probability map that is top-$k$ masked and repaired by the same exact-assignment routine, so only the candidate generator differs from NCP.
\end{description}

\subsection{Maximum Independent Set (MIS)}
\label{app:details:mis}

\paragraph{Instances and ground truth.}
We use \textsc{twitter-snap}, \textsc{collab}, and \textsc{imdb-binary}, under a $60{:}20{:}20$ train/validation/test split fixed by a single seed and shared across all methods. Optimal solutions are computed by an exact integer programming solver; every graph is solved to optimality. For the vertex-weighted variant, node weights $\omega_v\sim\mathcal U(0,1)$ are drawn once with a fixed seed so that all methods see identical inputs, and ground truth is recomputed under $\max\sum_v\omega_v x_v$.

\paragraph{Clique-cover certificate.}
The structure certificate $y$ is a clique cover of $G$, constructed once per graph by an edge-incremental greedy procedure that repeatedly extends an uncovered edge to a maximal clique and removes its edges from the uncovered set. The cover is verified post-construction. It is not learned and is shared across methods that consume it; it indexes the structured set
\[
    \mathcal S_G(y)
    \;=\;
    \bigl\{x\in[0,1]^V : \textstyle\sum_{v\in C}x_v\le 1\;\forall C\in y\bigr\},
\]
which strictly contains the stable-set polytope and is a tighter relaxation than the edge formulation.

\paragraph{Outer network for $\Gamma_\theta(G)$.}
$\Gamma_\theta(G)\in\mathbb R^{|V|}$ is produced by a multi-layer multi-head GIN backbone following \cite{karalias2020erdos}, projected through a small head and squashed to $[-1,1]$ around a zero-perturbation anchor. It enters the inner layer through node weights $w_v=1+\Gamma_\theta(G)_v$ for unit MIS and $w_v=\omega_v+\Gamma_\theta(G)_v$ for the weighted variant, so $\Gamma_\theta(G)$ acts as a learned reweighting of the node objective.

\paragraph{Inner recovery layer.}
$Q_G(\Gamma,y,\lambda)$ is $T$ steps of projected primal--dual ascent on the clique-cover relaxation $\max\{\langle w,x\rangle : x\in\mathcal S_G(y)\}$, with primal updates clipped to $[0,1]$ and dual updates rectified at zero. All $T$ steps are unrolled, so $\hat\mu$ is fully differentiable in $\theta$ via backpropagation through the recovery iterations.

\paragraph{Decoder.}
The certificate readout $C_G$ is a greedy independent set: nodes are visited in descending $\hat\mu$ order and added to the IS unless they are adjacent to an already-selected node. The weighted variant scores the same procedure by $\sum_{v\in S}\omega_v$.

\paragraph{Training.}
We train under the unsupervised relaxation loss
\begin{equation*}
    \mathcal L
    \;=\; -\langle w_{\text{ref}},\hat\mu\rangle
    \;+\; \beta\sum_{(u,v)\in E}\mathrm{ReLU}(\hat\mu_u+\hat\mu_v-1),
    \qquad \beta=0.05,
\end{equation*}
where $w_{\text{ref}}$ is the unit weight (resp.\ the true node weight $\omega$ for the weighted variant), \emph{not} the perturbed weight. Optimization uses Adam with a step decay schedule. Checkpoints are selected by mean validation IS size (resp.\ IS weight).

\paragraph{Baselines.}
\begin{description}
\item[Erdos-Neural \cite{karalias2020erdos}.] Same GIN backbone with the original Bernoulli-style relaxation loss $-\sum_v p_v+\beta\sum_{(u,v)\in E}p_u p_v$; per-dataset $\beta$ chosen by validation. No inner recovery layer.
\item[\textsc{DIFUSCO} \cite{sun2023difusco}.] Categorical discrete diffusion over node assignments, supervised by the indicator of the exact optimum; standard greedy projection at decoding.
\item[\textsc{Fast-T2T} \cite{li2024fast}.] Same encoder as \textsc{DIFUSCO}, trained under a consistency objective and decoded in
one step plus a single rewrite.
\item[\textsc{KaMIS}~\cite{lamm2016finding}.] An external exact branch-and-cut MIS solver under a $5$-second time limit; included as a non-learning reference.
\end{description}

\subsection{Elementary Shortest Path with Negative Cycles (ESPP-NC)}
\label{app:details:espp}

\paragraph{Instances.}
We synthesize $3{,}000$ directed graphs equally split among Erdős--Rényi, Barabási--Albert each with $|V|\in \{50,100,200,500\}$,
fixed source/sink, and edge weights drawn so that negative cycles are common; densities and orientations are randomized within fixed ranges. Each graph is decorated with spectral, positional, and degree node features, and edge features comprising weight and a spectral distance. We use an $80{:}20$ split per family.

\paragraph{Certificate and outer network.}
The certificate $y$ for ESPP-NC is a node-potential vector $d\in\mathbb R^{|V|}$ that turns the Bellman recursion well-posed in the presence of negative cycles. The output GNN plays the role of $\Gamma_\theta(G)=\Delta c_{uv}$ as the edge weight perturbation.

\paragraph{Inner recovery layer.}
$Q_G(\Gamma,y,\lambda)$ is $K=20$ iterations of a discounted smoothed Bellman update
\begin{align*}
    [T_\tau d]_u
    &\;=\; -\tau\,\mathrm{logsumexp}_{v:(u,v)\in E}
            \!\bigl((-c_{uv}-d_v)/\tau\bigr), \\
    d^{(k+1)}
    &\;=\; \gamma\,T_\tau d^{(k)} + (1-\gamma)\,d^{(0)},
    \qquad d^{(k+1)}_t = 0,
\end{align*}
with $\gamma=0.9$ and disabled edges masked. Certificate consistency reduces to $y=C_G(\hat\mu)$ with $C_G$ defined implicitly through the smoothed Bellman fixed point.

\paragraph{Primal recovery.}
Given the fixed point $d^{\star}$, the recovered transition law is the softmax of the Bellman residual over each node's out-edges with the sink absorbing,
\[
    P_{u,v}
    \;=\;
    \mathrm{softmax}_v\!\bigl((-c_{uv}-d^{\star}_v)/\tau\bigr),
    \qquad
    P_{t,\cdot}=0.
\]
Expected node visits under the resulting absorbing chain are obtained by solving a small regularized linear system, and the recovered
marginal $\hat\mu_{uv}=x_u P_{u,v}$ is the expected edge flow. The regularizer keeps the system well-conditioned as $\tau\to 0$ and the transition law sharpens.

\paragraph{Loss and training.}
Training minimizes the unsupervised expected path cost $\mathcal L=\langle c,\hat\mu\rangle$ with no auxiliary penalties on a mixed ER and BA graph set including only 100-node size. Optimization uses Adam with learning rate $5\!\times\!10^{-4}$, batch size $32$, $30$ epochs, and gradient-norm clipping; a per-batch loss floor discards rare cycle-collapsed batches before they corrupt the gradient.

\paragraph{$\tau$-annealing.}
At fixed $\tau>0$ the smoothed Bellman map is a strict contraction, but the recovered transitions retain mass on edges that violate $d^{\star}$ as a potential, so probability flow can re-enter cycles. We linearly anneal $\tau$ from $\tau_{\text{start}}=5.0$ to $\tau_{\text{end}}=0.2$ over the $30$ training epochs. As $\tau$ decreases, the recovered transitions concentrate on edges that respect $d^{\star}$ as a potential and elementarity of the support emerges from the certificate without an explicit penalty term. 

\paragraph{Decoder.}
At inference, the recovered transition law $P$ is fed to a path decoder that samples up to $300$ elementary $s$--$t$ paths by drawing
successors from a softmax over unvisited out-edges weighted by $\log P_{u,v}$, with the visited set enforcing elementarity and the
path length capped at $|V|$. The decoder returns the cheapest completed path under the original cost $c$.

\paragraph{Baselines.}
\begin{description}
\item [\textsc{Labeling}~\cite{feillet2004exact}.]  A labeling algorithm with the standard dominance rule
\[
    \text{cost}(A)\le\text{cost}(B)\;\wedge\;
    \text{visited}(A)\subseteq\text{visited}(B)
    \;\Rightarrow\;A\text{ dominates }B,
\]
under a label-count budget with 1M labels, using \texttt{julia} implemented.
\item[\textsc{LP-Decode}~\cite{chen2025unsupervised}.] Solve the bounded shortest-path LP with standard flow-conservation constraints and greedily extend a
partial path from $s$ along edges with positive LP support, preferring larger support and breaking ties by lower cost. 
\item[\textsc{Beam Search}~\cite{lowerre1976harpy}.] Run a width $\{256,512\}$ beam search over elementary paths under a time budget $= 30\,s$ without count-budget, returning the completed path with the lowest cost.
\end{description}
\section{Out-of-distribution Results on MIS}
\label{Appendix:Out}
\begin{table}[h]
\centering
\caption{Cross-dataset zero-shot transfer results. Bold indicates best OOD performances.}
\label{tab:mis-zeroshot}
\scriptsize
\setlength{\tabcolsep}{12pt}
\begin{tabular}{l l ccc}
\toprule
& & \multicolumn{3}{c}{Test dataset} \\
\cmidrule(lr){3-5}
Method & Trained on & TWITTER & COLLAB & IMDB \\
\midrule
\multirow{3}{*}{Erdos}
  & TWI   & 0.9558 / 1.81\,ms  & 0.8179 / 1.02\,ms & 0.5450 / 0.16\,ms \\
  & COL  & 0.9697 / 1.86\,ms          & 0.9976 / 0.91\,ms & 0.9229 / 0.10\,ms \\
  & IMDB   & 0.5597 / 1.74\,ms          & 0.4271 / 0.93\,ms & 0.4373 / 0.11\,ms \\
\midrule
\multirow{3}{*}{DIFUSCO}
  & TWI & 0.9809 / 303\,ms  & 0.9908 / 369\,ms          & 0.9972 / 272\,ms          \\
  & COL  & 0.9775 / 301\,ms           & 0.9980 / 323\,ms & 1.0000 / 281\,ms          \\
  & IMDB    & 0.9376 / 287\,ms           & 0.9669 / 299\,ms          & 0.9967 / 266\,ms \\
\midrule
\multirow{3}{*}{Fast-T2T}
  & TWI & 0.9876 / 30.6\,ms & \textbf{0.9992} / 26.3\,ms         & 1.0000 / 22.8\,ms         \\
  & COL  & 0.9800 / 30.7\,ms          & 0.9992 / 26.7\,ms& 1.0000 / 22.5\,ms         \\
  & IMDB    & 0.9510 / 31.0\,ms          & 0.9881 / 27.2\,ms         & 0.9978 / 22.7\,ms\\
\midrule
\multirow{3}{*}{NCP}
  & TWI & 0.9881 / 1.93\,ms & 0.9973 / 0.98\,ms         & \textbf{1.0000 }/ 0.20\,ms         \\
  & COL  & \textbf{0.9891} / 1.87\,ms          & 0.9984 / 0.97\,ms& \textbf{1.0000} / 0.18\,ms         \\
  & IMDB   & \textbf{0.9635} / 1.97\,ms   & \textbf{0.9908} / 1.03\,ms         & 1.0000 / 0.14\,ms\\
\bottomrule
\end{tabular}
\\[2pt]
\end{table}

\end{document}